%% file: binary_lth.tex
\documentclass{article}

\usepackage[accepted]{icml2021}

\usepackage[utf8]{inputenc} %
\usepackage[T1]{fontenc}    %
\usepackage{hyperref}       %
\usepackage{url}            %
\usepackage{booktabs}       %
\usepackage{amsfonts}       %
\usepackage{nicefrac}       %
\usepackage{microtype}      %

\usepackage{amsmath, amssymb}
\usepackage{amsthm}
\usepackage{bigstrut}
\usepackage{float}
\usepackage{enumitem}
\usepackage{graphicx}

\usepackage[font=footnotesize]{caption}
\usepackage{mathtools}
\usepackage{subfig}
\usepackage{wrapfig}

\usepackage{dsfont}
\usepackage{mathtools}
\usepackage{multirow}
\usepackage{xfrac}
\input{math_commands.tex}

\usepackage{soul}
\usepackage{colortbl}
\definecolor{mydarkblue}{rgb}{0,0.08,0.45}
\definecolor{mydarkgreen}{rgb}{0,0.45,0.08}
  \hypersetup{ %
    colorlinks=true,
    linkcolor=mydarkblue,
    citecolor=mydarkblue,
    filecolor=mydarkblue,
    urlcolor=mydarkblue}

\usepackage{cuted}
\usepackage{flushend}
\usepackage{balance}

\newcommand{\sm}{supplementary materials}
\newcommand\op[1]{\operatorname{#1}}

\newcommand{\eg}{{\it e.g.}, }
\newcommand{\ie}{{\it i.e.}, }

\makeatletter
\def\@fnsymbol#1{\ensuremath{\ifcase#1\or w \or k \or \dagger\or \mathsection\or
   \ddager\or \mathparagraph\or \|\or **\or \dagger\dagger
   \or \ddagger\ddagger \else\@ctrerr\fi}}
\makeatother

\renewcommand{\epsilon}{\varepsilon}

\usepackage{bm}
\usepackage{bbm}

\newtheorem{theorem}{Theorem}
\newtheorem{proposition}{Proposition}
\newtheorem{inf_theorem}{Theorem}
\newtheorem{inf_proposition}{Proposition}

\newtheorem{lemma}{Lemma}
\newtheorem{fact}{Fact}
\newtheorem{remark}{Remark}
\newtheorem{corollary}{Corollary}

\date{}
\begin{document}

\twocolumn[
\icmltitle{Finding Everything within Random Binary Networks}

\begin{icmlauthorlist}
\icmlauthor{Kartik Sreenivasan}{uw}
\icmlauthor{Shashank Rajput}{uw}
\icmlauthor{Jy-yong Sohn}{uw}
\icmlauthor{Dimitris Papailiopoulos}{uw}
\end{icmlauthorlist}

\icmlaffiliation{uw}{University of Wisconsin-Madison}

\icmlcorrespondingauthor{Kartik Sreenivasan}{ksreenivasa2@wisc.edu}

\icmlkeywords{Machine Learning, ICML}

\vskip 0.3in
]

\printAffiliationsAndNotice{}

\begin{abstract}
A recent work by \citet{RamanujanEtAl20} provides significant empirical evidence that sufficiently overparameterized, random neural networks contain untrained subnetworks that achieve state-of-the-art accuracy on several predictive tasks. A follow-up line of theoretical work provides justification of these findings by proving that slightly overparameterized neural networks, with commonly used continuous-valued random initializations can indeed be pruned to approximate any target network. In this work, we show that the amplitude of those random weights does not even matter. We prove that any target network can be approximated up to arbitrary accuracy by simply pruning a random network of binary $\{\pm1\}$ weights that is only a polylogarithmic factor wider and deeper than the target network.
\end{abstract}

\section{Introduction}
As the number of parameters of state-of-the-art networks continues to increase, pruning has become a prime choice for sparsifying and compressing a model.
A rich and long body of research, dating back to the 80s, shows that one can prune most networks to a tiny fraction of their size, while maintaining high accuracy 
\citep{mozer1989skeletonization,Stork93,LevinEtAl94,LeCunEtAl90,han2015learning, han2015deep,li2016pruning, wen2016learning,hubara2016binarized,hubara2017quantized,he2017channel,wu2016quantized,zhu2016trained, he2018amc,zhu2017prune,ChengEtAl19,mlsys2020_73,deng2020model}.

A downside of most of the classic pruning approaches is that they sparsify a model once it is trained to full accuracy, followed by significant fine-tuning, resulting in  a computationally burdensome procedure. 
\citet{frankle2018lottery}
conjectured the existence of {\it lottery tickets}, \ie sparse subnetworks at (or near) initialization, that can be trained---just once---to reach the accuracy of state-of-the-art dense models. This may help alleviate the computational burden of prior approaches, as training is predominantly carried on a much sparser model. The conjectured existence of these lucky tickets is referred to as the Lottery Ticket Hypothesis (LTH).
\citet{frankle2018lottery} and \cite{frankle2020linear} show that not only do {\it lottery tickets exist}, but also the cost of ``{\it winning the lottery}'' is not very high. 

Along the LTH literature, a curious phenomenon was observed; even at initialization and in the {\it complete absence} of training, one can find sub-networks of the random initial model that have prediction accuracy far beyond random guessing~\citep{zhou2019deconstructing, RamanujanEtAl20, WangEtAl19}. \citet{RamanujanEtAl20} reported this  in its most striking form:
state-of-the-art accuracy models for CIFAR10 and ImageNet, {\it simply reside} within slightly larger, yet completely random networks, and appropriate pruning---and mere pruning---can reveal them!
This ``{\it pruning is all you need}'' phenomenon  is sometimes referred to as the Strong Lottery Ticket Hypothesis.

A recent line of work attempts to establish the theoretical validity of the Strong LTH by studying the following non-algorithmic question:
\begin{quotation}
\begin{center}
    \noindent \it Can a random network be pruned to approximate a  target function $f(x)$?
\end{center}
\end{quotation}
Here, $f$ represents a bounded range labeling function that acts on inputs $x\in\mathcal{X}$, and is itself a neural network of finite width and depth. This assumption is not limiting, as neural networks are universal approximators \citep{stinchombe1989universal, barron1993universal, scarselli1998universal, klusowski2018approximation, perekrestenko2018universal, hanin2019universal, kidger2020universal}.
Note that the answer to the above question is trivial if one does not constraint the size of the random initial network, for all interesting cases of ``\emph{random}''. Indeed, if we start with an exponentially wider random neural network compared to the one representing $f$, by sheer luck, one can always find weights, for each layer, near-identical to those of any target neural network that is $f$. 
Achieving this result with a constrained overparameterization, \ie the degree by which the random network to be pruned is wider/deeper than $f$, is precisely why this question is challenging.

\citet{malach2020proving} were the first to prove that the Strong LTH is true, assuming polynomial-sized overparameterization. Specifically, under some mild assumptions, they showed that to approximate a target network of width $d$ and depth $l$ to within error $\epsilon$, it suffices to prune a random network of width $\widetilde{\cO}(d^2 l^2/\epsilon^2)$ and depth $2l$.
\citet{pensia2020optimal} offered an exponentially tighter bound using a connection to the SubsetSum problem. They showed that to approximate a target network within error $\epsilon$, it is sufficient to prune a randomly initialized network of width $\cO(d\log(dl/\epsilon))$ and depth $2l$. A corresponding lower bound for constant depth networks was also established.
\citet{orseau2020logarithmic} were also able to reduce the dependence on $\epsilon$ to logarithmic. They show that in order to approximate a target network within error $\epsilon$, it suffices to prune a random network of width $\cO(d^2\log(dl/\epsilon))$ if the weights are initialized with the hyperbolic distribution. However, this bound on overparamaterization is still polynomial in the width $d$.

The above theoretical studies have focused exclusively on continuous distribution for initialization. However, in the experimental work by \cite{RamanujanEtAl20}, the authors manage to obtain the best performance by pruning networks of scaled, {\it binary weights}.
Training binary networks has been studied extensively in the past~\citep{courbariaux2015binaryconnect, simons2019review} as they are compute, memory and hardware efficient, though in many cases they suffer from significant loss of accuracy. The findings of \cite{RamanujanEtAl20} suggest that the accuracy loss may not be fundamental to  networks of binary weights, when such networks are learned by pruning.
Arguably, since ``\emph{carving out}'' sub-networks of random models is expressive enough to approximate a target function, \eg according to \cite{pensia2020optimal,malach2020proving}, one is posed to wonder about the importance of weights altogether.
So perhaps, {\it binary weights is all you need}.

\citet{diffenderfer2021multi} showed that indeed scaled binary networks can be pruned to approximate any target function. The required overparameterization is similar to that of \citet{malach2020proving}, \ie polynomial in the width, depth and error of the approximation. Hence in a similar vein to the improvement that \citet{pensia2020optimal} offered over the bounds of \citet{malach2020proving}, we explore whether such an improvement is possible on the results of \citet{diffenderfer2021multi}.

\paragraph{Our Contributions:} In this work, we offer an exponential improvement to the theoretical bounds by \cite{diffenderfer2021multi}, establishing the following.
\begin{inf_theorem} (informal)\label{thm:infThm1}
Consider a randomly initialized, FC, binary $\{\pm1\}$ network of ReLU activations, with depth $\Theta\left(l\log{(\frac{d l}{\epsilon})}\right)$ and width $\Theta\left(d \log^2{(\frac{d l}{\epsilon\delta})}\right)$, with the last layer consisting of scaled binary weights $\{\pm C\}$.
Then, there is always a constant $C$ such that this network can be pruned to approximate \textbf{any} FC ReLU network, up to error $\epsilon>0$ with depth $l$ and width $d$, with probability at least $1-\delta$.
\end{inf_theorem}

Therefore, we show that in order to approximate any target network, it suffices to just prune a logarithmically overparameterized binary network (Figure~\ref{fig:concept}). In contrast to  \citet{diffenderfer2021multi}, our construction only requires that the last layer be scaled while the rest of the network is purely binary $\{\pm1\}$.  We show a comparison of the known Strong LTH results in Table~\ref{tab:comparison}.
 
\begin{table*}[t]
  \centering
  \resizebox{2\columnwidth}{!}{%
  \begin{tabular}{|l|l|l|l|l|}
    \hline
    \textbf{Reference} & \textbf{Width} & \textbf{Depth} & \textbf{Total Params} &  \textbf{Weights} \\ \hline
		 \citet{malach2020proving}		&		$\widetilde{\cO}(d^2 l^2/\epsilon^2)$		&		$2l$       &     $\widetilde{\cO}(d^2 l^3/\epsilon^2)$         &    Real\\ \hline
		 
		 \citet{orseau2020logarithmic}		&		$\cO(d^2\log(dl/\epsilon)$		&		$2l$		&      $\cO(d^2l\log(dl/\epsilon)$       &				Real (Hyperbolic)\\
		 \hline
		 
		 \citet{pensia2020optimal} &   $\cO(d\log(dl/\min\{\epsilon, \delta\})$      &   $2l$       &        $\cO(dl\log(dl/\min\{\epsilon, \delta\})$      &      Real\\
		 \hline
		 
		 \citet{diffenderfer2021multi}		&		$\cO((ld^{3/2}/\epsilon) + ld\log(ld/\delta))$		&		$2l$		&      $\cO((l^2d^{3/2}/\epsilon) + l^2d\log(ld/\delta))$       &				$\{\pm \epsilon\}$\\
		 \hline
		 
         \textbf{Ours}, Theorem~\ref{thm:thm1}		&		$\cO(d\log^2{(dl/\epsilon\delta)})$      &		$\cO(l\log{(dl/\epsilon)})$		&      $\cO(dl\log^3{(dl/\epsilon\delta)})$       &		Binary-$\{\pm1\}$\footnotemark \\ \hline
    \end{tabular}
    }
\caption{Comparing the upper bounds for the overparameterization needed to approximate a target network (of width $d$ and depth $l$) within error $\epsilon > 0$ with probability at least $1-\delta$ by pruning a randomly initialized network.}
\label{tab:comparison}
\end{table*}

\begin{figure}[t]
\vspace{-3mm}
  \centering
  \includegraphics[width=0.5\textwidth]{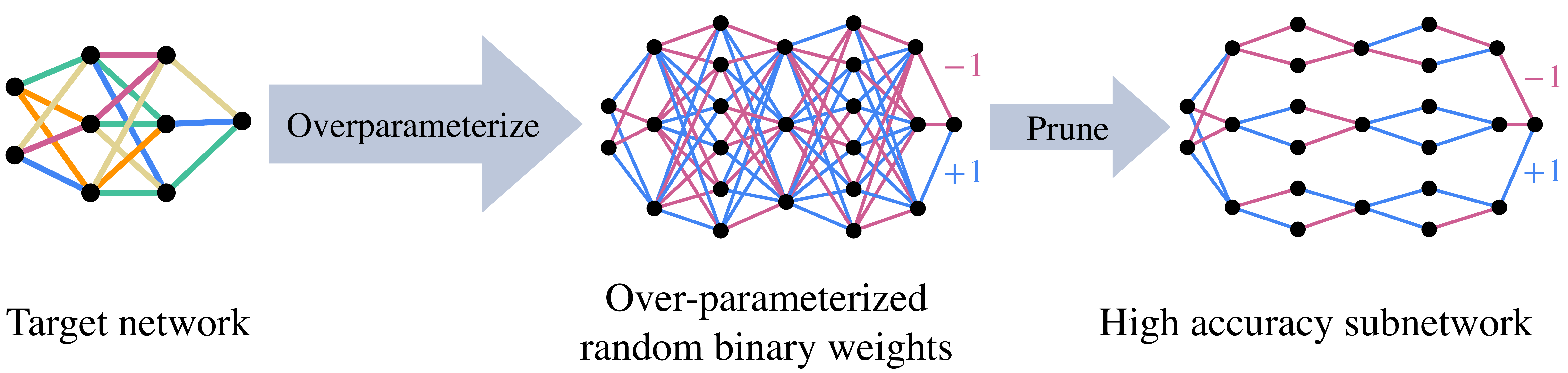}
\textbf{}  \caption{Approximating a target network with high accuracy by pruning overparameterized random binary network. In this paper, we show that logarithmic overparameterization in both width and depth is sufficient.}
  \label{fig:concept}
  \vspace{-3mm}
\end{figure}

In light of our theoretical results, one may wonder why in the literature of training, \ie assigning a sign pattern to fixed architecture, binary networks, a loss of accuracy is observed, \eg \cite{rastegari2016xnor}. Is this an algorithmic artifact, or does pruning random signs offer higher expressivity than assigning the signs?
We show that  there exist target functions that can be well approximated by pruning binary networks, yet none of all possible, binary, fully connected networks can approximate it.

\begin{inf_proposition} (informal)\label{prop:infProp2}
There exist a function $f$ that can be represented by pruning a random 2-layer binary network of width $d$, but not by any  2-layer fully-connected binary network of width $d$.
\end{inf_proposition}

Note that although finding a subnetwork of a random binary network results in a ``ternary'' architecture (e.g., 0 becomes a possible weight), the total number of possible choices of subnetworks is $2^N$, if $N$ is the total number of weights. This is equal to the total number of sign assignments of the same FC network. Yet, as shown in the proposition above, pruning a random FC network is provably more expressive than finding a sign assignment for the same architecture.

\section{Preliminaries and Problem Setup}\label{sec:Prelim}

Let $f(\vx): \mathbb{R}^{d_0} \rightarrow \mathbb{R}$ be the target FC network with $l$ layers and ReLU activations, represented as
\begin{equation*}
    f(\vx) = \sigma(\mW_l\sigma(\mW_{l-1} \dots \sigma(\mW_1\vx))),
\end{equation*}
where $\vx \in \mathbb{R}^{d_0}$ is the input, $\sigma(\vz) = \max\{\vz, 0\}$ is the ReLU activation and $\mW_i \in \mathbb{R}^{d_{i} \times d_{i-1}}$ is the weight matrix of layer $i \in [l]$. With slight abuse of terminology, we will refer to $f$ as a network, as opposed to a labeling function.
We then consider a binary\footnotemark[1] network of depth $l'$
\begin{equation*}
    g(\vx) = \sigma((\epsilon'\mB_{l'})\sigma(\mB_{l'-1} \dots \sigma(\mB_1\vx))),
\end{equation*}
where $\mB_i \in \{-1, +1\}^{d_{i}' \times d_{i-1}'}$ is a binary weight matrix, with all weights drawn uniformly at random from $\{\pm 1\}$, for all layers $i \in [l']$ and the last layer is multiplied by a factor of $\epsilon'>0$. The scaling factor is calculated precisely in Section~\ref{subsec:putting_things_together}, where we show that it is unavoidable for function approximation (\ie regression), rather than classification.

Our goal is to find the smallest network $g$ so that it contains a subnetwork $\tilde{g}$ which approximates $f$ well. More precisely, we will bound the overparameterization of the binary network, under which one can find supermask matrices $\mM_i \in \{0,1\}^{d_{i}' \times d_{i-1}'}$, for each layer $i \in [l']$, such that the pruned network
\begin{align*}
    \tilde{g}({\vx}) =
    &\sigma(\varepsilon'({\mM}_{l'}\odot{\mB}_{l'})\sigma(({\mM}_{l'-1}\odot{\mB}_{l'-1})\ldots\\
    &\ldots\sigma(({\mM}_{1}\odot{\mB}_{1}){\vx})))
\end{align*}
is $\epsilon$-close to $f$ in the sense of uniform approximation over the unit-ball, \ie 
\begin{equation*}
\label{eq:epsApprox}
    \max_{\vx \in \mathbb{R}^{d_0}: ||\vx|| \leq 1} ||f(\vx) - \tilde g(\vx)|| \leq \epsilon
\end{equation*}
for some desired $\epsilon > 0$.
In this paper, we show $g$ only needs to be polylogarithmically larger than the target network $f$ to have this property.
We formalize this and provide a proof in the following sections.

Henceforth, we denote $[k] = \{1, 2, \cdots, k \}$ for some positive integer $k$. Unless otherwise specified, $||\cdot||$ refers to the $\ell_2$ norm. We also use the max norm of a matrix, defined as $||\mA||_{\text{max}} := \max_{ij} |A_{ij}|$. The element-wise product between two matrices $\mA$ and $\mB$ is denoted by $\mA \odot \mB$. We assume without loss of generality that the weights are specified in the base-10 system. However, since we don't specify the base of the logarithm explicitly in our computations, we use the $\Theta(\cdot)$ notation to hide constant factors that may arise from choosing different bases.

\footnotetext[1]{The weights of all the layers are purely binary $\{\pm1\}$ except for the last layer which is scaled so that it is $\{\pm\epsilon'\}$ where $\epsilon' = (\eps/d^2l)^l$.}

\section{Strong Lottery Tickets by Binary Expansion}
\label{sec:intUpperBound}
In this section, we formally present our approximation results. We show that in order to approximate any target network $f(\vx)$ within arbitrary approximation error $\eps$, it suffices to prune a random binary\footnotemark[1] network $g(\vx)$ that is just polylogarithmically deeper and wider than the target network.

\subsection{Main Result}

First, we point out that the scaling factor $\eps'$ in the final layer of $g(\vx)$ is necessary for achieving arbitrary small approximation error for any target network $f(\vx)$. In other words, it is impossible to approximate any arbitrary target network with a purely binary $\{\pm1\}$ network regardless of the overparameterization. To see this, note that for the simple target function $f(x) = \epsilon x,\; x \in [0, 1]$ and $\epsilon \in [0.5, 1)$, the best approximation possible by a binary network is $g(x) = x$ and therefore $\max_{x \in \mathbb{R}: |x| \leq 1} |f(x) - g(x)| \geq (1-\epsilon)$ for any binary network $g$. We will show that just by allowing the weights of the final layer to be scaled, we can provide a uniform approximation guarantee while the rest of the network remains binary $\{\pm1\}$. Formally, we have the following theorem:

\begin{theorem}
\label{thm:thm1}
Consider the set of FC ReLU networks $\mathcal{F}$ defined as
\begin{align*}
    \mathcal{F} = \{f: f(\vx) = \sigma(\mW_l\sigma(\mW_{l-1} \dots \sigma(\mW_1\vx))), \\\forall i\; \mW_i \in \mathbb{R}^{d_i\times d_{i-1}}\; ||\mW_i|| \leq 1\},
\end{align*}
and let $d=\max_{i} d_i$. For arbitrary target approximation error $\eps$, let $g(\vx) = \sigma(\epsilon'\mB_{l'}\sigma(\mB_{l'-1} \dots \sigma(\mB_1\vx)))$ (here $\epsilon'=(\eps/d^2 l)^l$) be a randomly initialized network with depth $l' = \Theta(l\log({d^2l/\epsilon}))$ such that every weight is drawn uniformly from $\{-1, +1\}$ and the layer widths are $\Theta\left(\log{d^2l/\epsilon}\cdot\log\left(\frac{dl\log^2{(d^2l/\epsilon)}}{\delta}\right)\right)$ times wider than $f(\vx)$.

Then, with probability at least $1-\delta$, for every $f \in \mathcal{F}$, there exist pruning matrices $\mM_i$ such that
\begin{equation*}
    \max_{\vx \in \mathbb{R}^{d_0}: ||\vx|| \leq 1} |f(\vx) - \tilde g(\vx)| \leq \epsilon
\end{equation*}
holds 
where 
\begin{align*}
    \tilde{g}(\vx) := &\sigma(\eps' ({\mM}_{l'}\odot {\mB}_{l'}) \sigma(({\mM}_{l'-1}\odot{\mB}_{l'-1})\ldots\\
    &\ldots \sigma(({\mM}_{1}\odot{\mB}_{1}){\vx}))).
\end{align*}
\end{theorem}

\begin{remark}
The dimensions of the weight matrices of $g(\vx)$ in Theorem~\ref{thm:thm1} are specified more precisely below.
Let $p=(d^2l/\epsilon)$.
Since $l' = l \log (p)$, we have $\lfloor\log(p)\rfloor$ layers in $g(\vx)$ that approximates each layer in $f(\vx)$. For each $i \in [l]$, the dimension of $\mB_{(i-1)\lfloor\log(p)\rfloor+1}$ is
\begin{equation*}
    \Theta\left(d_{i-1} \log(p)\log\left(\frac{dl\log^2{(p)}}{\delta}\right)\right) \times d_{i-1},
\end{equation*}
the dimension of $\mB_{i\lfloor\log(p)\rfloor}$ is
\begin{equation*}
    d_{i} \times \Theta\left(d_{i-1} \log{(p)}\log\left(\frac{dl\log^2{(p)}}{\delta}\right)\right)
\end{equation*}
and the remaining $\mB_{(i-1)\lfloor\log(p)\rfloor+k}$ where $1 < k < \lfloor \log(p) \rfloor$ have the dimension
\begin{align*}
    \Theta\left(d_{i-1} \log{(p)}\log\left(\frac{dl\log^2{(p)}}{\delta}\right)\right)\\
    \;\times \Theta\left(d_{i-1} \log{(p)}\log\left(\frac{dl\log^2{p}}{\delta}\right)\right).
\end{align*}

\end{remark}

\subsection{Proof of Theorem~\ref{thm:thm1}}

First, we show in Section~\ref{sec:log_precision_enough} that any target network in $f(\vx) \in \mathcal{F}$ can be approximated within $\epsilon > 0$, by another network $\hat{g}_{p}(\vx)$ having weights of finite-precision at most $p$ digits where $p$ is logarithmic in $d, l,$ and $\epsilon$.

Then, in Section~\ref{sec:bin_weights_enough}, we show that any finite precision network can be represented exactly using a binary network where all the weights are binary ($\pm1$) except the last layer, and the last layer weights are scaled-binary ($\pm \epsilon'$).
The proof sketch is as follows. First, through a simple scaling argument we show that any finite-precision network is equivalent to a network with integer weights in every layer except the last. We then present Theorem~\ref{thm:thm2} which shows the deterministic construction of a binary network using diamond-shaped gadgets that can be pruned to approximate any integer network.
Lemma~\ref{lem:networkIntRandApprox} extends the result to the case when the network is initialized with random binary weights.

Putting these together completes the proof of Theorem~\ref{thm:thm1} as shown in Section~\ref{subsec:putting_things_together}.

\subsubsection{Logarithmic precision is sufficient}
\label{sec:log_precision_enough}

First, we consider the simplest setting wherein the target network contains a single weight \ie $h(x) = \sigma(wx)$, where $x, w$ are scalars, the absolute values of which are bounded by $1$. This assumption can be relaxed to any finite norm bound.
We begin with noting a fact that $\log(1/\eps)$ digits of precision are sufficient to approximate a real number within error $\eps$, as formalized below  %
\begin{fact}\label{lem:precision}
Let $w\in \mathbb{R},\; |w|\leq 1$ and $\hat{w}$ be a finite-precision truncation of $w$ with $\ceil{\Theta(\log(1/\eps))}$ digits. Then $|w-\hat{w}| \leq \eps$ holds.
\end{fact}

Now we state the result for the case when the target network contains a single weight $w$.
\begin{lemma}\label{lem:scalarRealApprox}
Consider a network $h(x) = \sigma(wx)$ where $w \in \mathbb{R}, |w| \leq 1$. For a given $\eps>0$, let $\hat{w}$ be a finite-precision truncation of $w$ up to $\log(1/\epsilon)$ digits and let $ \hat{g}_{\log(1/\epsilon)}(x) = \sigma(\hat{w} x)$. Then we have
\begin{equation*}
    \max_{x \in \mathbb{R}: |x| \leq 1} |h(x) - \hat{g}_{\log(1/\epsilon)}(x)| \leq \epsilon.
\end{equation*}

\end{lemma}
\begin{proof}
By Fact~\ref{lem:precision}, we know that $|w - \hat{w}| \leq \epsilon$. Applying Cauchy-Schwarz with $|x| \leq 1$ gives us $|\hat{w}x - wx| \leq \epsilon$. Since this holds for any $x$ and ReLU is 1-Lipschitz, the result follows.
\end{proof}

Lemma~\ref{lem:scalarRealApprox} can be extended to show that it suffices to consider finite-precision truncation up to $\log(d^2l/\epsilon)$ digits to approximate a network for width $d$ and depth $l$. This is stated more formally 
below.

\begin{lemma}\label{lem:networkRealApprox}
Consider a network $h(\vx) = \sigma(\mW_l\sigma(\mW_{l-1} \dots \sigma(\mW_1\vx)))$ where $\mW_i \in \mathbb{R}^{d_i\times d_{i-1}}$, $||\mW_i|| \leq 1$. For a given $\epsilon > 0$, define $\hat{g}_{\log(d^2l/\epsilon)}(\vx) = \sigma(\widehat{\mW}_l\sigma(\widehat{\mW}_{l-1} \dots \sigma(\widehat{\mW}_1\vx)))$ where $\widehat{\mW_i}$ is a finite precision truncation of $\mW_i$ up to $\log(d^2l/\epsilon)$ digits, where $d = \max_i d_i$. Then we have
\begin{equation*}
    \max_{\vx \in \mathbb{R}^{d_0}: ||\vx|| \leq 1} |h(\vx) - \hat{g}_{\log(d^2l/\epsilon)}(\vx)| \leq \epsilon.
\end{equation*}
\end{lemma}

We provide the proof of Lemma~\ref{lem:networkRealApprox} as well as approximation results for a single neuron and layer in \sm.

\subsubsection{Binary weights are sufficient}
\label{sec:bin_weights_enough}
We begin by showing that any finite-precision FC ReLU network can be represented perfectly as a FC ReLU network with integer weights in every layer except the last, using a simple scaling argument. Since ReLU networks are positive homogenous, we have that $\sigma(c\cdot z) = c\cdot\sigma(z)$ for $c>0$. Given a network $g_{p}$ where all the weights are of finite-precision at most $p$, we can apply this property layerwise with the scaling factor $c=10^p$ so that,
\begin{align}
    f(\vx) &= \sigma(\mW_l\sigma(\mW_{l-1} \dots \sigma(\mW_1\vx))) \nonumber\\
    &= \frac{1}{c^l}\sigma(c\mW_l\sigma(c\mW_{l-1} \dots \sigma(c\mW_1\vx))) \nonumber\\
    &= \sigma(c'\widehat{\mW}_l\sigma(\widehat{\mW}_{l-1} \dots \sigma(\widehat{\mW}_1\vx))) \label{eq:scalingTrick}
\end{align}
where $\widehat{\mW}_i = 10^p \mW_i$ is a matrix of integer weights and $c' = \frac{1}{c^l}$. Therefore, the rescaled network has integer weights in every layer except the last layer which has the weight matrix $c'\widehat{\mW}_l = (c^{-l})\mW_l$.

In the remaining part of this section, we show that any FC ReLU network with integer weights can be represented exactly by pruning a purely binary $(\pm1)$ FC ReLU network which is just polylogarithmic wider and deeper. More precisely, we prove the following result.
\begin{theorem}
\label{thm:thm2}
Consider the set of FC ReLU networks with integer weights $\mathcal{F}_W$ defined as
\begin{align*}
    \mathcal{F}_W = \{f: f(\vx) = \sigma(\mW_l\sigma(\mW_{l-1} \dots \sigma(\mW_1\vx))),\\
    \forall i\; \mW_i \in \mathbb{Z}^{d_i\times d_{i-1}}\; ||\mW_i||_{max} \leq W\}
\end{align*}
where $W > 0$. Define $d = \max_i d_i$ and let $g(\vx) = \sigma(\mB_{l'}\sigma(\mB_{l'-1} \dots \sigma(\mB_1\vx)))$ be a network with depth $l' = \Theta(l\log({|W|}))$ where every weight is uniform-randomly generated from $\{-1, +1\}$ and the layer widths are $\Theta\left(\log{|W|}\cdot\log\left(\frac{dl\log^2{|W|}}{\delta}\right)\right)$ times wider than $f(\vx)$.

Then, with probability at least $1-\delta$, for every $f \in \mathcal{F}$, there exist pruning matrices $\mM_i$ such that
\begin{equation*}
    f(\vx) = \tilde{g}(\vx)
\end{equation*}
holds for any $\vx \in \mathbb{R}^{d_0}$ where $\tilde{g}(\vx) := \sigma(({\mM}_{l'}\odot {\mB}_{l'}) \sigma(({\mM}_{l'-1}\odot{\mB}_{l'-1})\ldots\sigma(({\mM}_{1}\odot{\mB}_{1}){\vx})))$.
\end{theorem}

\begin{remark}
The dimensions of the weight matrices of $g(\vx)$ in Theorem~\ref{thm:thm2} are specified more precisely below.
Note that we have $\lfloor\log{|W|}\rfloor$ layers in $g(\vx)$ that exactly represents each layer in $f(\vx)$.
For each $i \in [l]$, the dimension of  $\mB_{(i-1)\lfloor\log{|W|}\rfloor+1}$ is
\begin{equation*}
    \Theta\left(d_{i-1} \log{|W|}\log\left(\frac{dl\log^2{|W|}}{\delta}\right)\right) \times d_{i-1},
\end{equation*}
the dimension of $\mB_{i\lfloor\log{|W|}\rfloor}$ is
\begin{equation*}
    d_{i} \times \Theta\left(d_{i-1} \log{|W|}\log\left(\frac{dl\log^2{|W|}}{\delta}\right)\right)
\end{equation*}
and the remaining $\mB_{(i-1)\lfloor \log |W| \rfloor + k}$ where $1 < k < \lfloor \log |W| \rfloor$ have the dimension
\begin{align*}
    \Theta\left(d_{i-1} \log{|W|}\log\left(\frac{dl\log^2{|W|}}{\delta}\right)\right)\\
    \times \Theta\left(d_{i-1} \log{|W|}\log\left(\frac{dl\log^2{|W|}}{\delta}\right)\right)
\end{align*}
\end{remark}

\begin{remark}
Note that $\tilde g(\vx)$ is \textbf{exactly} equal to $f(\vx)$. Furthermore, we provide a uniform guarantee for \emph{all} networks in $\mathcal{F}$ by pruning a single over-parameterized network, like~\citet{pensia2020optimal}.
\end{remark}

\begin{remark}
Theorem~\ref{thm:thm2} can be made into a deterministic construction for any fixed target network thereby avoiding the $\log(1/\delta)$ overparameterization. We extend to the random initialization setting by resampling the construction a sufficient number of times.
\end{remark}

\begin{remark}
To resolve issues of numerical overflow, we can insert scaling neurons after every layer.
\end{remark}

\begin{remark}
The integer assumption can easily be converted to a finite-precision assumption using a simple scaling argument. Since all networks in practice use finite-precision arithmetic, Theorem~\ref{thm:thm2} may be of independent interest to the reader. However, we emphasize here that there is no approximation error in this setting. Practitioners who are interested in small error($10^{-k}$) can just apply Theorem~\ref{thm:thm1} and incur an overparameterization factor of $\cO(k)$.
\end{remark}

The proof of Theorem~\ref{thm:thm2} will first involve a deterministic construction for a binary network that gives us the desired guarantee. We then extend to the random binary initialization. The construction is based on a diamond-shaped gadget that allows us to approximate a single integer weight by pruning a binary ReLU network with just logarithmic overparameterization.

First, consider a target network that contains just a single integer weight \ie $h(x) = \sigma(wx)$. We will show that there exists a binary FC ReLU network $g(x)$ which can be pruned to approximate $h(x)$.

\begin{lemma}\label{lem:scalarIntApprox}
Consider a network $h(x) = \sigma(wx)$ where $w \in \mathbb{Z}$. Then there exists a FC ReLU binary network $g(x)$ of width and depth $O(\log|w|)$ that can be pruned to $\tilde{g}(x)$ so that $\tilde{g}(x) = h(x)$ for all $x \in \mathbb{R}$.
\end{lemma}

\emph{Proof. }%
Note that since $w$ is an integer, it can be represented using its binary (base-$2$) expansion
\begin{equation}
\label{eq:bin_exp}
    w = \sign(w) \sum_{k=0}^{\lfloor \log_2 |w| \rfloor} z_k \cdot 2^k, \quad \; z_k \in \{0, 1\}.
\end{equation}

For ease of notation, we will use $\log(\cdot)$ to represent $\log_2(\cdot)$ going forward.
Denote $n = \lfloor \log_2 |w| \rfloor$.
The construction of $g_n(x)$ in Figure~\ref{fig:2a} shows that $2^n$ can be represented by a binary network with ReLU activations for any $n$. We will refer to this network as the diamond-shaped ``gadget''.

Note that the expansion in Equation~(\ref{eq:bin_exp}) requires $2^k$ for all $0 \leq k \leq n =  \lfloor\log|w|\rfloor$. Luckily, any of these can be represented by just pruning $g_n(x)$ as shown in Figure~\ref{fig:2b}.

\begin{figure}[t]
    \vspace{-2mm}
	\centering
	\subfloat[][\centering{$g_n(x) = 2^n \max\{0, x\}$}]{\includegraphics[height=18mm ]{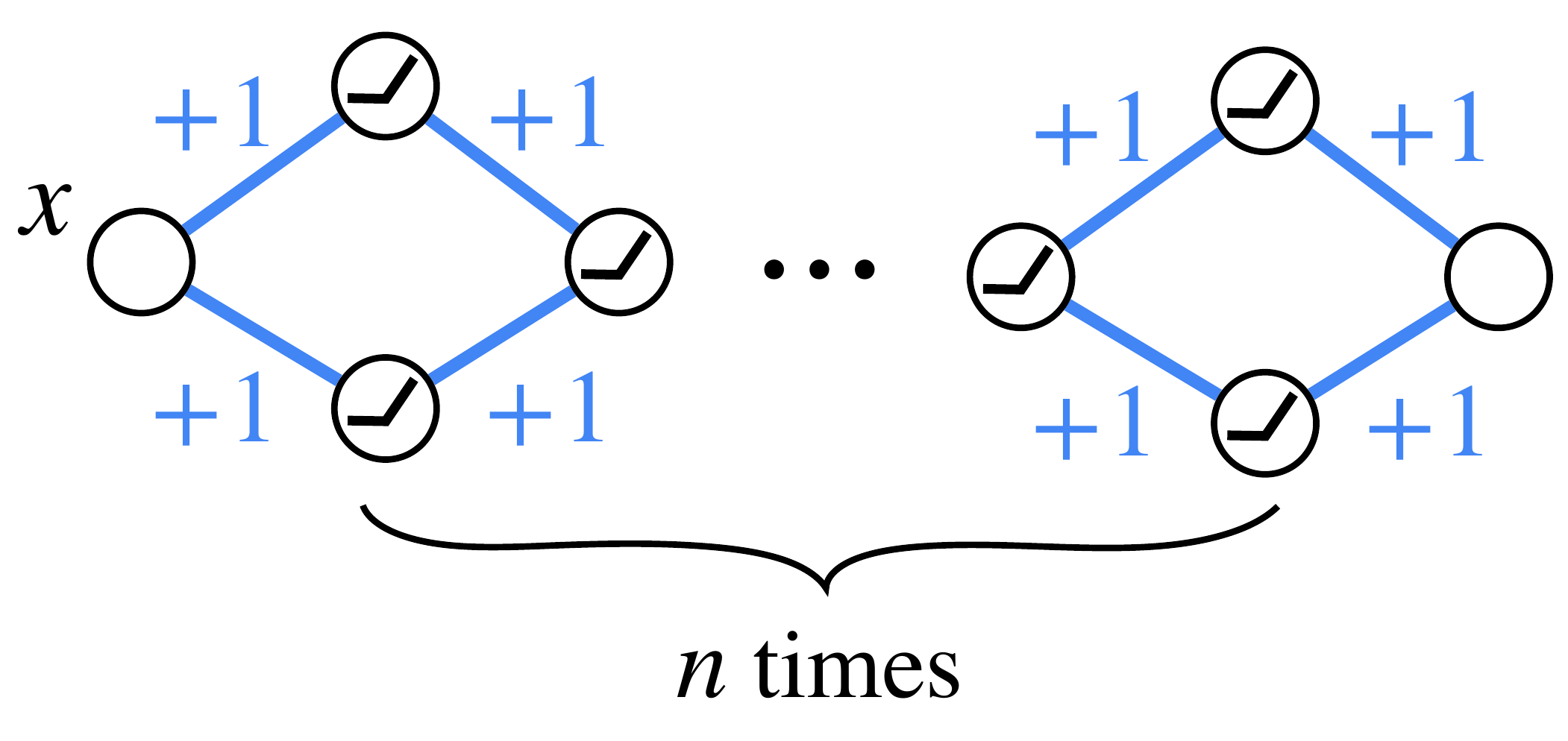}\label{fig:2a}}
	\quad 
 	\subfloat[][\centering{$g_{n-k}(x) = 2^{n-k} \max\{0, x\}$} ]{\includegraphics[height=18mm ]{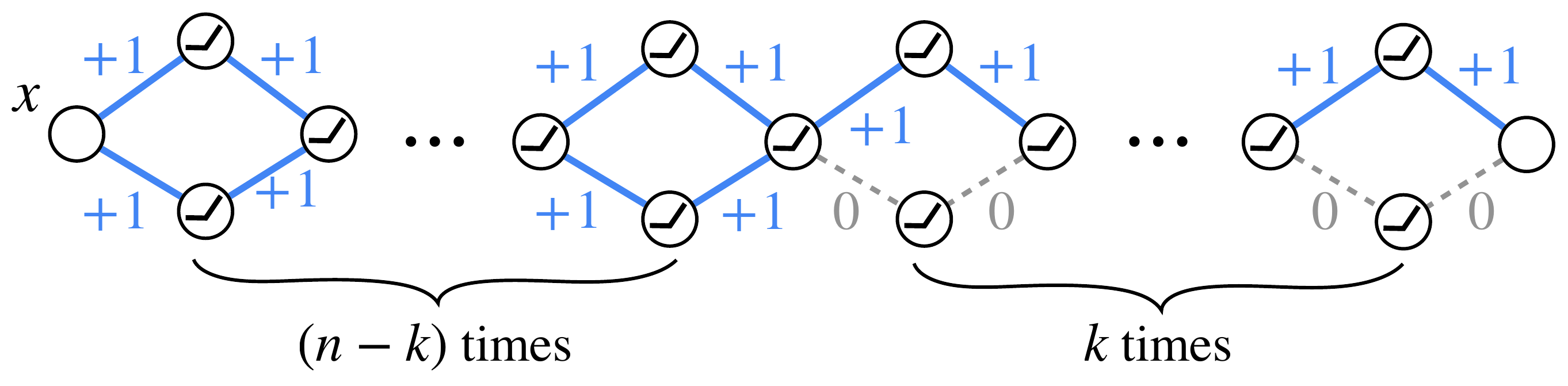}\label{fig:2b}} 
	\vspace{-2mm}
	\caption
	{The diamond-shaped binary ReLU networks that compute $g_n(x)$ and $g_{n-k}(x)$, respectively. The dashed edges are just weights that have been ``pruned'' (set to $0$). The output neuron is a simple linear activation.
    }
	\label{fig:2}
\end{figure}

\begin{figure}[t]
    \vspace{-2mm}
	\centering
	\subfloat[][\centering{$f_{n-k}^{+}(x) = 2^{n-k} x$}]{\includegraphics[width=0.35\textwidth ]{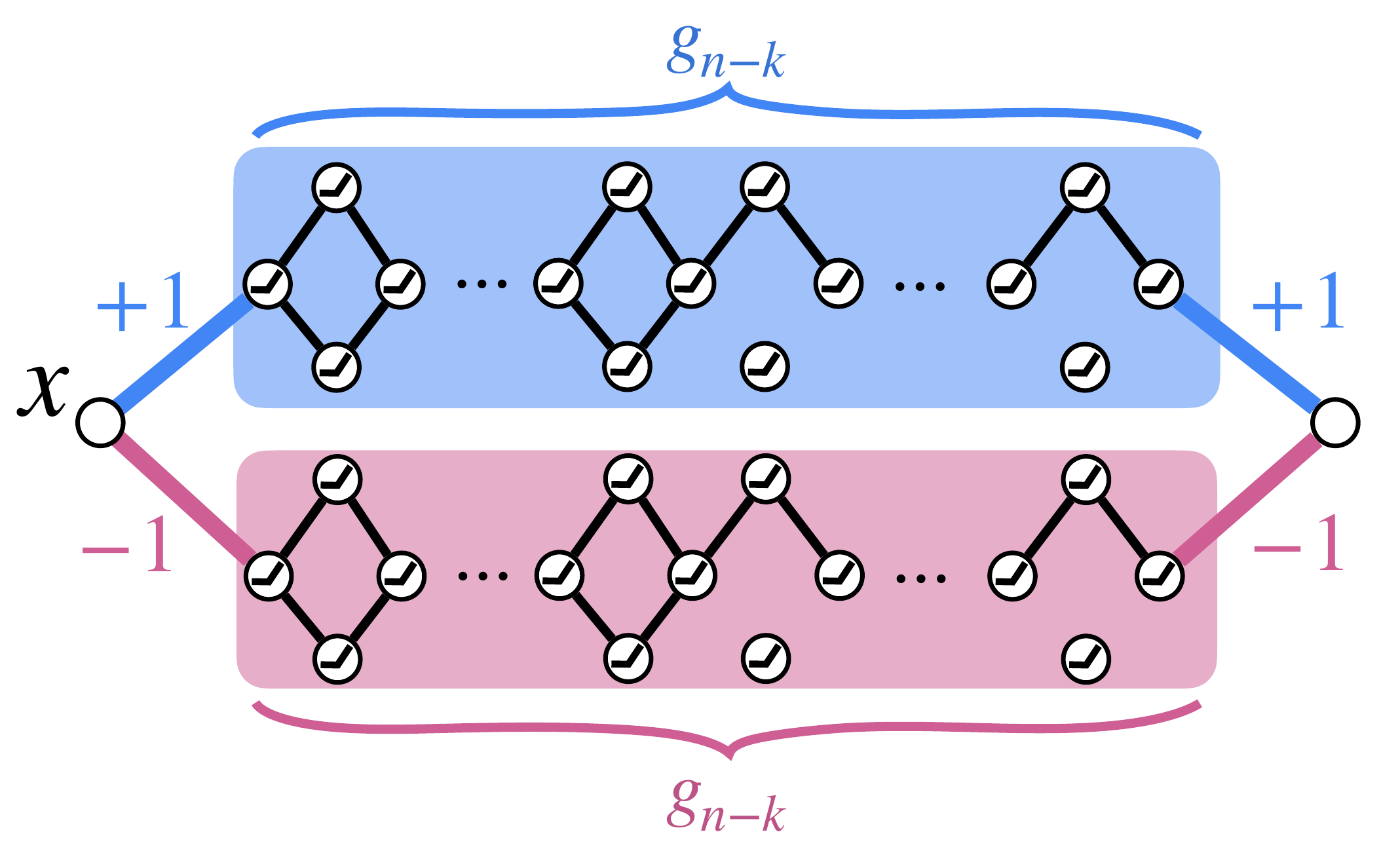}\label{fig:4a}}
	\quad 
 	\subfloat[][\centering{$f_{n-k}^{+}(x) = -2^{n-k} x$} ]{\includegraphics[width=0.35\textwidth]{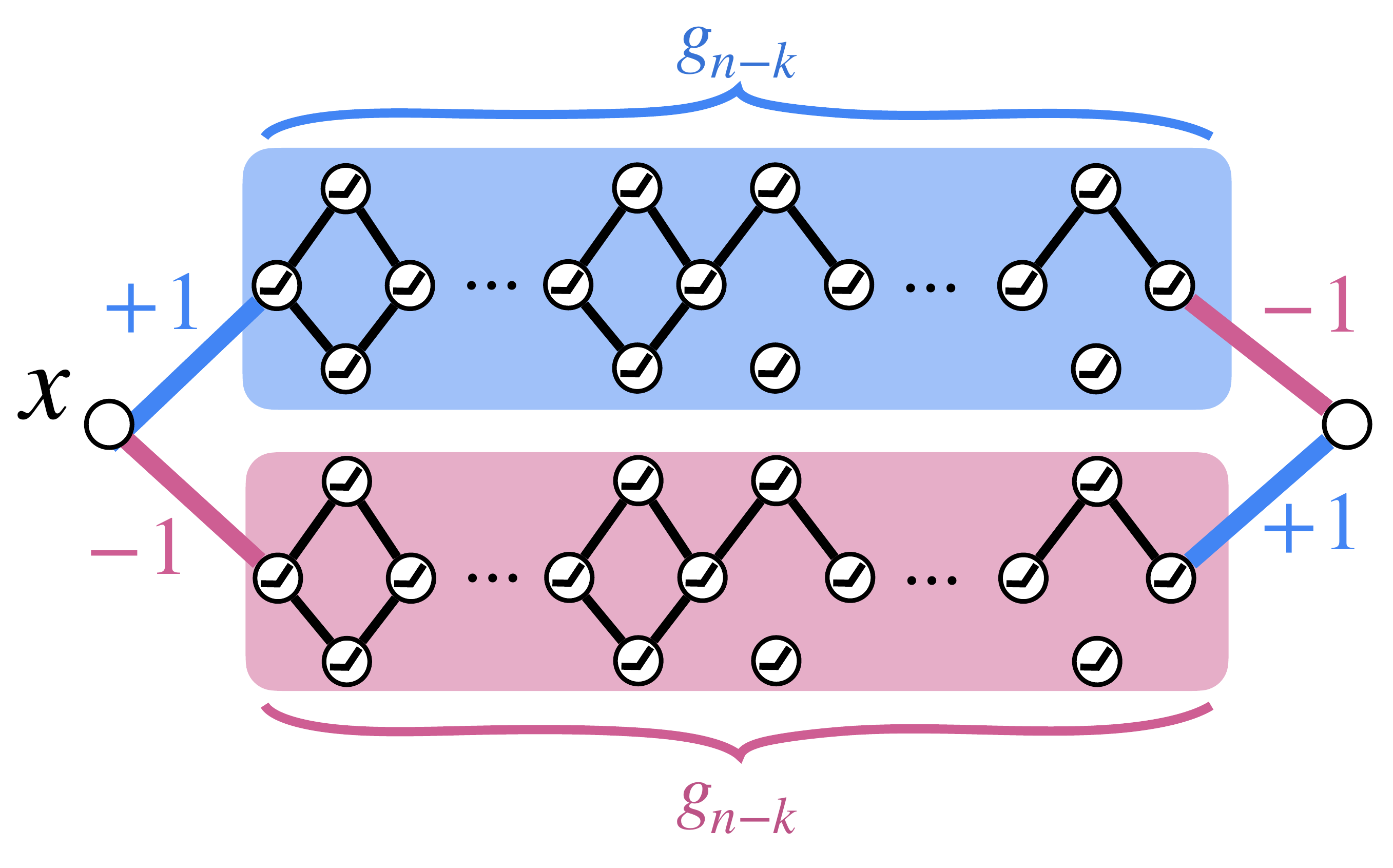}\label{fig:4b}} 
	\vspace{-2mm}
	\caption
	{We can use two instances of $g_{n-k}$ to create (a) $f_{n-k}^+$ and (b) $f_{n-k}^-$ to approximate both positive weights and negative weights. The output neuron here has a linear activation.}
	\label{fig:3}
\end{figure}

\begin{figure}[t]
	\centering
	\includegraphics[width=0.4\textwidth]{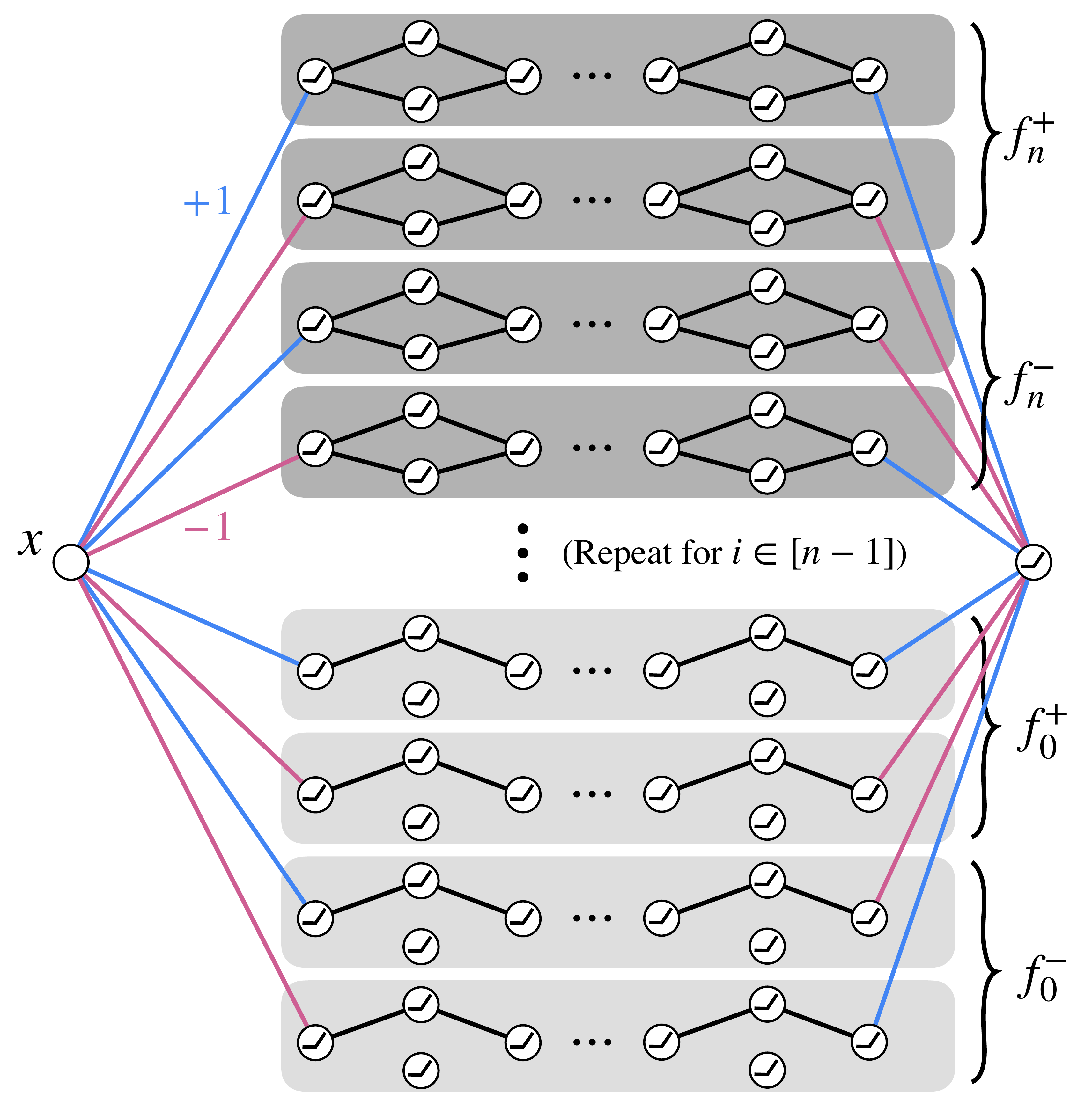}
	\caption{Illustration of 
	$g_k(x) = \sum_{k=0}^n (f_k^+(x) + f_k^-(x)) = \sigma(\pm \sum_{k=0}^n 2^k x)$ which can be further pruned to approximate $f(x) = \sigma(wx)$ for any $w: |w| < 2^{n+1} - 1$}
	\label{fig:binary_weight_approx}
\end{figure}

However, since our constructions use the ReLU activation, this only works when $x \geq 0$. Using the same trick as ~\cite{pensia2020optimal}, we can extend this construction by mirroring it as shown in Figure~\ref{fig:3}. This gives us $f_{n-k}^{+}(x) := 2^{n-k}x$ and $f_{n-k}^{-}(x) := -2^{n-k}x$. The correctness of this mirroring trick relies on the simple observation that $wx = \sigma(wx) - \sigma(-wx)$.

Putting these together, we get $g_n(x) = \sigma(\pm\sum_{k=0}^{n} 2^k x)$ as shown in Figure~\ref{fig:binary_weight_approx}. By pruning just the weights in the last layer, we can choose which terms to include. Setting $n=\lfloor \log|w| \rfloor$ completes the approximation.

To calculate the overparameterization required to approximate $h(x) = \sigma(wx)$, we simply count the parameters in the above construction. Each gadget $g_k$ is a network of width $2$ and depth $(\lfloor\log|w|\rfloor)$. To construct $f_k^{+}$, we need two such gadgets. Therefore to construct $f_k^{+}$ and $f_k^{-}$, we need width $4$ and depth $(\lfloor\log|w|\rfloor)$. Repeating this for each $k \in 1,2,\dots,\lfloor\log|w|\rfloor$ shows that our construction is a network of width and depth $\cO(\log|w|)$ which completes the proof of Lemma~\ref{lem:scalarIntApprox}.\qed

\begin{remark}
The network in Fig.~\ref{fig:binary_weight_approx} used for proving Lemma~\ref{lem:scalarIntApprox} can be written as
\begin{align*}
    g(x) =
    &\sigma(\mM_{\vv} \odot \vv)^T [({\mM}_{n}\odot {\mB}_{n})\sigma(({\mM}_{n-1}\odot{\mB}_{n-1})\ldots\\
    &\ldots \sigma(({\mM}_{1}\odot{\mB}_{1})\sigma(({\mM}_{\vu}\odot{\vu}){x})))]),
\end{align*}
where $\{\mM_i\}_{i \in [n]}, \mM_{\vv}, \mM_{\vu}$ are mask matrices and $\{\mB_i\}_{i \in [n]}, \vv, \vu$ are binary weight matrices. By pruning elements in $\vu$ or $\vv$, one can obtain $h(x) = \sigma(wx)$.
We will always prune the last layer $\vv$ as it makes the construction more efficient when we extend it to approximating a layer.
\end{remark}

Now, we extend the construction to the case where the target function is a neural network with a single neuron \ie $h(x) = \sigma(\vw^T \vx)$.

\begin{lemma}\label{lem:neuronIntApprox}
Consider a network $h(\vx) = \sigma(\vw^T \vx)$ where $\vw \in \mathbb{Z}^{d}$ and $||\vw||_{\infty} \leq w_{max}$. Then there exists a FC ReLU binary network $g(\vx)$ of width $\cO(d\log|w_{max}|)$ and depth $\cO(\log|w_{max}|)$ that can be pruned to $\tilde{g}(\vx)$ so that $\tilde{g}(\vx) = h(\vx)$ for all $\vx \in \mathbb{R}^{d}$.
\end{lemma}

\begin{proof}
A neuron can be written as $h(\vx) = \sigma(\vw^T \vx) = \sigma(\sum_{i=1}^d w_i x_i)$. Therefore, we can just repeat our construction from above for each $w_i, i\in [d]$. This results in a network of width $\cO(d\log|w_{max}|)$ while the depth remains unchanged at $\cO(\log|w_{max}|)$.
\end{proof}

\begin{figure}[t]
	\centering
	\includegraphics[width=0.45\textwidth]{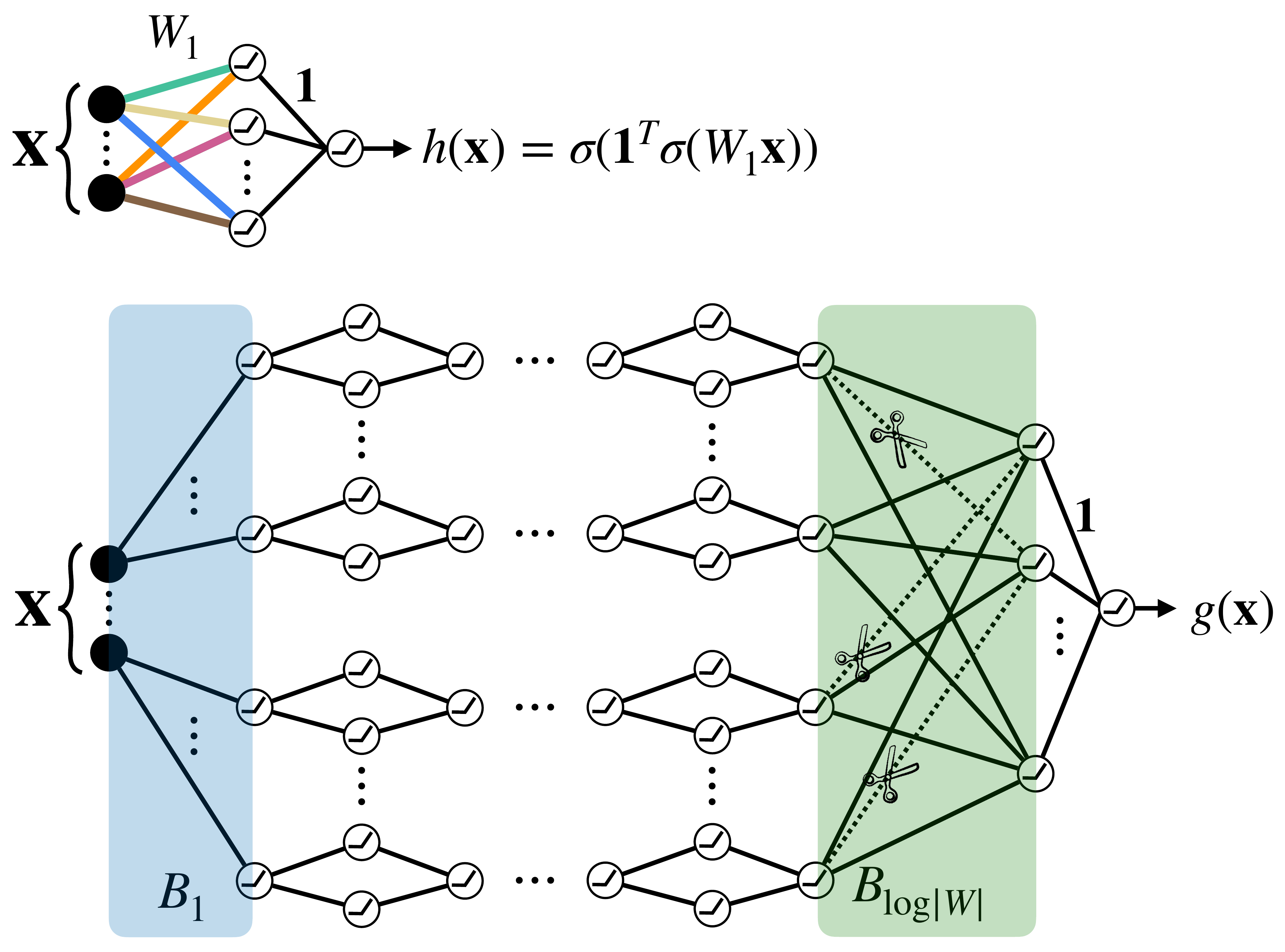}
	\caption{Illustration of the construction in Lemma~\ref{lem:layerIntApprox}: Approximating a 1-hidden layer network $h(\vx) = \sigma(\vOne^T \sigma(\mW_1 \vx))$ by pruning the appropriate binary network $g(\vx)$. Pruning the last layer allows us to ``re-use'' weights.}
	\label{fig:real_to_binary}
\end{figure}

Next, we describe how to approximate a single layer target network and avoid a quadratic overparameterization.

\begin{lemma}\label{lem:layerIntApprox}
Consider a network $h(\vx) = \sigma(\vOne^T \sigma(\mW_1 \vx))$ where $\mW_1 \in \mathbb{Z}^{d_1\times d_0}$ and $||\mW_1||_{max} \leq W$. Then there exists a FC ReLU binary network $g(\vx)$ of width $\cO(\max\{d_0, d_1\}\log|W|)$ and depth $\cO(\log|W|)$ that can be pruned to $\tilde{g}(\vx)$ so that $\tilde{g}(\vx) = h(\vx)$ for all $\vx \in \mathbb{R}^{d_0}$.
\end{lemma}

\begin{proof}
Note that $\vOne \in \mathbb{R}^{d_1}$ is the vector of $1$'s. Naively extending the construction from Lemma~\ref{lem:neuronIntApprox} would require us to replace each of the $d_0$ neurons in the first layer by a network of width $\cO(d_1\log|W|)$ and depth $\cO(\log|W|)$. This already needs a network of width $\cO(d_0d_1\log|W|)$ which is a \emph{quadratic} overparameterization. Instead, we take advantage of pruning only the last layer in the construction from Lemma~\ref{lem:scalarIntApprox} to re-use a sufficient number of weights and avoid the quadratic overparameterization. An illustration of this idea is shown in Figure~\ref{fig:real_to_binary}. The key observation is that by pruning only the last layer of each $f_n(x)$ gadget, we leave it available to be re-used to approximate the weights of the remaining $(d_{1} - 1)$ neurons. Therefore, the width of the network required to approximate $h(\vx)$ is just $\cO(\max\{d_0, d_1\}\log|W|)$ and the depth remains $\cO(\log|W|)$.
\end{proof}

Now we can tie it all together to show an approximation guarantee for any network in $\mathcal{F}$.
\begin{lemma}\label{lem:networkIntApprox}
Consider a network $f(\vx) = \sigma(\mW_l\sigma(\mW_{l-1} \dots \sigma(\mW_1\vx)))$ where $\mW_i \in \mathbb{Z}^{d_i\times d_{i-1}}$, $||\mW_i||_{max} \leq W$. Let $d =\max_i d_i$. Then, there exists a FC ReLU binary network $g(\vx)$ of width $\cO(d\log|W|)$ and depth $\cO(l\log|W|)$ that can be pruned to $\tilde{g}(\vx)$ so that $\tilde{g}(\vx) = f(\vx)$ for all $\vx \in \mathbb{R}^{d_0}$.
\end{lemma}

\begin{proof}
Note that there is no approximation error in any of the steps above. Therefore, we can just repeat the construction above for each of the $l$ layers in $f(\vx)$. This results in a network $g(\vx)$ of width $\cO(d\log|W|)$ and depth $\cO(l\log|W|)$. 
A more precise description of the dimensions of each layer can be found in the statement of Theorem~\ref{thm:thm1}.
\end{proof}

Finally, below we %
show that our construction can be extended for networks with randomly initialized weights. The proof can be found in \sm.
\begin{lemma}\label{lem:networkIntRandApprox}
Consider a network $f(\vx) = \sigma(\mW_l\sigma(\mW_{l-1} \dots \sigma(\mW_1\vx)))$ where $\mW_i \in \mathbb{Z}^{d_i\times d_{i-1}}$, $||\mW_i||_{max} \leq W$. Define $d = \max_i d_i$ and let $g(\vx)$ be a randomly initialized binary network of width $\Theta\left(d\log\left(\frac{(dl\log^2|W|)}{\delta}\right)\right)$ and depth $\Theta(l\log|W|)$ such that every weight is drawn uniformly from $\{-1, +1\}$. Then $g(\vx)$ can be pruned to $\tilde{g}(\vx)$ so that $\tilde{g}(\vx) = f(\vx)$ for all $\vx \in \mathbb{R}^{d_0}$ with probability at least $1-\delta$.
\end{lemma}

Since every network $f \in \mathcal{F}_W$ satisfies the assumptions of Lemma~\ref{lem:networkIntRandApprox}, we can apply it for the entire $\mathcal{F}_W$. Note that we do not need to apply the union bound over $\mathcal{F}_W$ since the randomness is just needed to ensure the existence of a particular deterministic network - specifically $g(x)$ from Lemma~\ref{lem:networkIntApprox}. Therefore, a single random network is sufficient to ensure the guarantee for \emph{every} $f \in \mathcal{F}_W$. This completes the proof of Theorem~\ref{thm:thm2}.

\subsubsection{Putting everything together}
\label{subsec:putting_things_together}
We now put the above results together to complete the proof of Theorem~\ref{thm:thm1}. First, note that by Lemma~\ref{lem:networkRealApprox}, to approximate any $f \in \mathcal{F}$ within $\epsilon >0$, it suffices to consider $\hat{g}(\vx)$ which is a finite-precision version of $f$ where the precision of each weight is at most $p = \log(d^2l/\epsilon)$. Now, applying the scaling trick in Equation (\ref{eq:scalingTrick})  we can represent $\hat{g}(\vx)$ exactly as a scaled integer network \ie
\begin{equation*}
    \hat{g}(\vx) = c^{-l}\sigma(c\widehat{\mW}_l\sigma(c\widehat{\mW}_{l-1} \dots \sigma(c\widehat{\mW}_1\vx)))
\end{equation*}
where $c = 10^p = (d^2l/\epsilon)$ and all the weight matrices $c\widehat{\mW}_i$ are integer. Since $\lVert \mW_i \rVert_{\op{max}} \leq 1$, it is clear that $\lVert c\widehat{\mW}_i \rVert_{max} \leq c$. Therefore, applying Theorem~\ref{thm:thm2} to the integer network $c^{l}\hat{g}(\vx)$ with $W=(d^2l/\epsilon)$, we have the following. If $h(\vx) = \sigma(\mB_{l'}\sigma(\mB_{l'-1} \dots \sigma(\mB_1\vx)))$ is a randomly initialized binary network of depth $\Theta(l\log(d^2l/\epsilon))$ and width $\Theta\left(\log(d^2l/\epsilon)\log\left(\frac{dl\log^2(d^2l/\epsilon)}{\delta}\right)\right)$, then with probability $1-\delta$, it can be pruned to $\tilde{h}(\vx)$ so that $c^l\hat{g}(\vx) = \tilde{h}(\vx)$ for any $\vx$ in the unit sphere. Therefore, to approximate $\hat{g}(\vx)$ we simply push the scaling factor $c^{-l}$ into the last layer $B_{l'}$ so that its weights are now scaled binary $\{\pm (\epsilon/d^2l)^l\}$. Combining this with the approximation guarantee between $\hat{g}(\vx)$ and $f(\vx)$ completes the proof.

\subsection{Binary weights for classification}\label{sec:classification}
Theorem~\ref{thm:thm2} can easily be extended for classification problems using finite-precision networks. Since $\sign(\cdot)$ is positive scale-invariant, we no longer even require the final layer to be scaled. Applying the same argument as Sec~\ref{subsec:putting_things_together} and then dropping the $c^{-l}$ factor gives us the following corollary.

\begin{corollary}\label{cor:binClassification}
Consider the set of binary classification FC ReLU networks $\mathcal{F}$ of width $d$ and depth $l$, where the weight matrices $\{\mW_i\}_{i=1}^{l}$ are of finite-precision at most $p$ digits. Let $g(\vx) = \sign(\mB_{l'}\sigma(\mB_{l'-1} \dots \sigma(\mB_1\vx)))$ be a randomly initialized binary network with depth $l' = \Theta(lp)$ and width $d'=\Theta(dp\log(dlp/\delta))$ such that every weight is drawn uniformly from $\{-1, +1\}$. Then, with probability at least $1-\delta$, for every $f \in \mathcal{F}$, there exist pruning matrices $\{\mM_i\}_{i=1}^{l'}$ such that $f(\vx) = \tilde{g}(\vx)$ for any $\vx$ where $\tilde{g}(\vx) := \sign(({\mM}_{l'}\odot {\mB}_{l'}) \sigma(({\mM}_{l'-1}\odot{\mB}_{l'-1})\ldots\sigma(({\mM}_{1}\odot{\mB}_{1}){\vx})))$.
\end{corollary}

\section{Conclusion}\label{sec:conc}
In this paper, we prove the Strong LTH for binary networks establishing that logarithmic overparameterization is sufficient for pruning algorithms to discover accurate subnetworks within random binary models. 
By doing this, we provide theory supporting the wide range of experimental work in the field, \eg scaled binary networks can achieve the best SOTA accuracy on benchmark image datasets~\citep{diffenderfer2021multi, RamanujanEtAl20}. Moreover, we show that only the last layer needs to be scaled binary, while  the rest of the network can be purely binary $\{\pm1\}$. It is well known in the binary network literature that a \emph{gain} term (scaling the weights) makes the optimization problem more tractable~\citep{simons2019review}. While this is known empirically, it would be interesting to study this from a theoretical perspective so we can identify better algorithms to find binary networks of high accuracy.

\section*{Acknowledgements}
The authors would like to thank Jonathan Frankle for early discussions on pruning random binary networks. This research was supported by ONR Grant N00014-21-1-2806.

\bibliography{references.bib}

\appendix
\section{Appendix}
\input{appendix}

\end{document}

%% file: math_commands.tex
\usepackage{amsmath,amsfonts,bm}
\usepackage{bbm}
\usepackage[mathscr]{eucal}

\def\eqref#1{equation~\ref{#1}}

\def\ceil#1{\lceil #1 \rceil}

\def\1{\bm{1}}

\def\eps{{\epsilon}}

\def\vOne{{\bm{1}}}

\def\vu{{\bm{u}}}
\def\vv{{\bm{v}}}
\def\vw{{\bm{w}}}
\def\vx{{\bm{x}}}

\def\vz{{\bm{z}}}

\def\mA{{\bm{A}}}
\def\mB{{\bm{B}}}

\def\mM{{\bm{M}}}

\def\mW{{\bm{W}}}

\DeclareMathAlphabet{\mathsfit}{\encodingdefault}{\sfdefault}{m}{sl}
\SetMathAlphabet{\mathsfit}{bold}{\encodingdefault}{\sfdefault}{bx}{n}

\def\cO{{\mathcal{O}}}

\DeclareMathOperator{\sign}{sign}

%% file: appendix.tex
\subsection{Proof of Lemma~\ref{lem:networkRealApprox}}

Before proving Lemma~\ref{lem:networkRealApprox} in its entirety, we first 
extend Lemma~\ref{lem:scalarRealApprox} to approximate a neuron with $\log(d/\epsilon)$-precision.
\begin{lemma}\label{lem:neuronRealApprox}
Consider a network $h(\vx) = \sigma(\vw^Tx)$ where $\vw \in \mathbb{R}^d, ||\vw|| \leq 1$ and $\epsilon > 0$. Let $\hat{\vw}$ be a coordinate-wise finite-precision truncation of $\vw$ up to $\log(d/\epsilon)$ digits and $\hat g(x) = \sigma(\hat{\vw} x)$. Then we have that
\begin{equation*}
    \max_{\vx \in \mathbb{R}^d: ||\vx|| \leq 1} ||h(\vx) - \hat g(\vx)|| \leq \epsilon.
\end{equation*}
\end{lemma}
\begin{proof}
Once again, by Lemma~\ref{lem:precision}, we have that for each coordinate $|\vw_i - \hat{\vw}_i| \leq \epsilon/d$. Therefore, $||\vw - \hat{\vw}|| \leq \sqrt{\sum_{i=1}^d \left(\frac{\epsilon}{d}\right)^2} \leq \epsilon$. Applying Cauchy-Schwarz with $||\vx||\leq 1$ completes the proof.
\end{proof}

We now extend the result to approximating a single layer with finite-precision.
\begin{lemma}\label{lem:layerRealApprox}
Consider a network $h(\vx) = \sigma(\vOne^T \sigma(\mW_1 \vx))$ where $\mW_1 \in \mathbb{R}^{d_1 \times d_0}, ||\mW_1|| \leq 1$ and $\epsilon > 0$. Let $\widehat{\mW_1}$ be a coordinate-wise finite-precision truncation of $\mW_1$ up to $\log(d^2/\epsilon)$ digits and $\hat g(x) = \sigma(\widehat{\mW_1} x)$ where $d = \max\{d_0, d_1\}$. Then we have that
\begin{equation*}
    \max_{\vx \in \mathbb{R}^{d_0}: ||\vx|| \leq 1} ||h(\vx) - \hat g(\vx)|| \leq \epsilon.
\end{equation*}
\end{lemma}
\begin{proof}
Note that for any $\vx: ||\vx|| \leq 1$,
\begin{align*}
    &\lVert h(\vx) - \tilde{g}(\vx)\rVert\\
    &= \lVert\sigma(\sum_{i=1}^{d_1} \sigma({{W_1}^{(i)}}x)) - \sigma(\sum_{i=1}^{d_1} \sigma({{\widehat{W_1}}^{(i)}}x))\rVert\\
    &\leq  \lVert\sum_{i=1}^{d_1} \sigma({{W_1}^{(i)}}x) - \sum_{i=1}^{d_1} \sigma({{\widehat{W_1}}^{(i)}}x)\rVert\tag*{(Since $\sigma$ is 1-Lipschitz)}\\
    &\leq \sum_{i=1}^{d_1} \lVert\sigma({{W_1}^{(i)}}x) - \sigma({{\widehat{W_1}}^{(i)}}x)\rVert\\
    &\leq \sum_{i=1}^{d_1} \frac{\epsilon}{\max\{d_0, d_1\}}\\
    &\leq \epsilon
\end{align*}
Since this holds for any $\vx$, it also holds for the $\vx$ that attains the maximum possible error which completes the proof.
\end{proof}
We are now ready to prove Lemma~\ref{lem:networkRealApprox} which states that to approximate any FC ReLU network within $\epsilon$, it suffices to simply consider weights with logarithmic precision.

\setcounter{lemma}{1}
\begin{lemma}\label{lem:networkRealApproxAppendix}
Consider a network $h(\vx) = \sigma(\mW_l\sigma(\mW_{l-1} \dots \sigma(\mW_1\vx)))$ where $\mW_i \in \mathbb{R}^{d_i\times d_{i-1}}$, $||\mW_i|| \leq 1, d_i \leq d$ and $\epsilon > 0$. Let $\hat{g}(\vx) = \sigma(\widehat{\mW}_l\sigma(\widehat{\mW}_{l-1} \dots \sigma(\widehat{\mW}_1\vx)))$ where $\widehat{\mW_i}$ is a finite precision truncation of $\mW_i$ up to $\log(d^2l/\epsilon)$ digits. Then we have that
\begin{equation*}
    \max_{\vx \in \mathbb{R}^{d_0}: ||\vx|| \leq 1} ||h(\vx) - \hat g(\vx)|| \leq \epsilon.
\end{equation*}
\end{lemma}
\begin{proof}
The proof follows inductively. First note that since operator norm is bounded by the frobenius norm, we have for the output of the first layer that,
\begin{align*}
    \lVert \sigma(\mW_1 \vx) - \sigma(\widehat{\mW}_1 \vx)\rVert &\leq \lVert \mW_1 \vx - \widehat{\mW}_1 \vx \rVert\\
    &\leq \lVert \mW_1 - \widehat{\mW}_1\rVert\\
    &\leq \lVert \mW_1 - \widehat{\mW}_1\rVert_F\\
    &\leq \epsilon/l
\end{align*}
Next, denote the output of the first layer by $\vx_2 := \sigma(\mW_1 \vx),\; \widehat{\vx}_2 := \sigma(\widehat{\mW}_1 \vx)$.
Repeating the calculation above for the second layer gives us,
\begin{align*}
    &\lVert \sigma(\mW_2\sigma(\mW_1 \vx)) - \sigma(\widehat{\mW}_2 (\sigma(\widehat{\mW}_1\vx))\rVert\\
    &=\lVert \sigma(\mW_2 \vx_2) - \sigma(\widehat{\mW}_2 \widehat{\vx}_2)\rVert\\
    &= \lVert \sigma(\mW_2 x_2) - \sigma(\widehat{\mW}_2 x_2) + \sigma(\widehat{\mW}_2 x_2) - \sigma(\widehat{\mW}_2 \widehat{\vx}_2)\rVert\\
    &\leq \lVert \sigma(\mW_2 x_2) - \sigma(\widehat{\mW}_2 x_2)\rVert + \lVert\sigma(\widehat{\mW}_2 x_2) - \sigma(\widehat{\mW}_2 \widehat{\vx}_2)\rVert
\end{align*}
Note that the first term is bounded by $\epsilon/l$ since $||\vx_2|| \leq 1$. Further, using the fact that $\lVert\widehat{\mW}_2\rVert \leq 1$ and $\lVert \vx_2 - \widehat{\vx}_2\rVert \leq \epsilon/l$ as proved above,
we get that
\begin{align*}
    \lVert \sigma(\mW_2\sigma(\mW_1 \vx)) - \sigma(\widehat{\mW}_2 (\sigma(\widehat{\mW}_1\vx))\rVert \leq 2\epsilon/l
\end{align*}
By induction, for each layer $1 \leq i \leq l$ we have that for any $\vx: \lVert \vx \rVert \leq 1$,
\begin{align*}
    &\lVert \big(\sigma(\mW_i\sigma(\mW_{i-1} \dots \sigma(\mW_1\vx)))\big)\\
    &- \big(\sigma(\widehat{\mW}_i\sigma(\widehat{\mW}_{i-1} \dots \sigma(\widehat{\mW}_1\vx)))\big)\rVert \leq \epsilon\cdot (i/l)
\end{align*}
Setting $i=l$ completes the proof.
\end{proof}
\setcounter{lemma}{6}

\begin{lemma}\label{lem:networkIntRandApproxAppendix}
Consider a network $f(\vx) = \sigma(\mW_l\sigma(\mW_{l-1} \dots \sigma(\mW_1\vx)))$ where $\mW_i \in \mathbb{Z}^{d_i\times d_{i-1}}$, $||\mW_i||_{max} \leq W$ and $d_i \leq d$. Let $g(\vx)$ be a randomly initialized binary network of width $\Theta\left(d\log\left(\frac{(dl\log^2|W|)}{\delta}\right)\right)$ and depth $\Theta(l\log|W|)$ such that every weight is drawn uniformly from $\{-1, +1\}$. Then $g(\vx)$ can be pruned to $\tilde{g}(\vx)$ so that $\tilde{g}(\vx) = f(\vx)$ for all $\vx \in \mathbb{R}^{d_0}$ with probability at least $1-\delta$.
\end{lemma}
\begin{proof}
Consider any one particular diamond structure in Figure~\ref{fig:2}. If we choose the weights of this network uniformly at random from $\{-1, +1\}$, then the probability that they are all $1$ is at most $(1/2)^4$. If we make the network $k$ times wider, the probability that such a diamond-gadget does not exist in the network is given by $(1-(1/2)^4)^k$. In other words, if we define the event $A$ to be the failure event, then $P(A) = \left(1- (1/2)^4\right)^k$. Note that there are $\Theta(dl\log^2|W|)$ such diamond structures in our network. By symmetry, the failure probability of each of them is identical. To have the overall probability of failure to be at most $\delta$, taking the union bound we have that
\begin{equation*}
    dl \log^2|W| \cdot \left(1 - (1/2)^4\right)^k \leq \delta
\end{equation*}
Hence, it suffices to have $k \geq \Theta\left(\log\left(\frac{(dl\log^2|W|)}{\delta}\right)\right)$. In other words, a randomly initialized binary network of width $\Theta\left(\log\left(\frac{(dl\log^2|W|)}{\delta}\right)\right)$ and depth $\Theta(l\log{|W|})$ contains our deterministic construction with probability at least $1-\delta$.
\end{proof}

\subsection{Observation on the power of zero}
We try to understand here if pruning is essential to the expressive power of binary networks. We note that pruning binary networks is strictly more powerful than merely sign-flipping their weights. To see this, consider the set of all networks that can be derived by pruning 1-hidden layer binary networks of width $m$ and input dimension $d$:
\begin{align*}
    \mathcal{F}_{pruned}^{m} = \bigg\{f : f(\vx) = \sigma\left(\sum_{i=1}^d x_i \sigma\left(\sum_{j=1}^m v_j w_j^{i}\right)\right),\\
    w_j, v_j \in \{+1, -1, 0\}\bigg\}.
\end{align*}
Similarly, the set of all 1-hidden layer binary networks of the same width without pruning is given by
\begin{align}
\label{eq:fBinDef}
    \mathcal{F}_{bin}^{m} = \bigg\{f : f(\vx) = \sigma\left(\sum_{i=1}^d x_i \sigma\left(\sum_{j=1}^m v_j w_j^{i}\right)\right),\nonumber\\
    w_j, v_j \in \{+1, -1\}\bigg\}.
\end{align}
We prove the following proposition that shows that $\mathcal{F}_{pruned}$ is a strictly richer class of functions indicating that pruning is an essential part of approximating classifiers.

\begin{proposition}\label{prop:Prop2}
The function $f(\vx) = \sigma\left(\sum_{i=1}^d i\cdot x_i\right)$ satisfies
$f(\vx) \in \mathcal{F}_{pruned}^{d}$ and $f(\vx) \notin \mathcal{F}_{bin}^{d}$, \ie without pruning, $f(\vx)$ cannot be represented by a single layer binary network of width $d$.
\end{proposition}
\begin{proof}
For simplicity, we consider the case when $x \geq 0$ so that the ReLU is equivalent to a linear activation. If the functions are not equal for the non-negative orthant, then they are surely different. First, note that we recover $f(\vx)$ from $\mathcal{F}_{pruned}^{d}$  if we set $v_j = 1$ and $w_j^{i} = \mathds{1}_{j \leq i}$ for all $i, j$. Therefore, $f(\vx) \in \mathcal{F}_{pruned}^{d}$ holds.
To see that $f \notin \mathcal{F}_{bin}^d$, first note that we can replace $v_jw_j^{i}$ with $z_j^{i} \in \{+1, -1\} \;\forall i, j$ in Equation~(\ref{eq:fBinDef}). Hence, any $g \in \mathcal{F}_{bin}^d$ is of the form $g(\vx) = \sigma\left(\sum_{i=1}^d x_i \sigma\left(\sum_{i=1}^d z_j^i\right)\right)$. Consider the coefficient of $x_d$ in $f(\vx)$ which is $(d-1)$. Since $z_j^{d}$ cannot be set to $0$, the best approximation we can get using $g(\vx)$ is $d$ or $d-2$. In fact, this holds for every odd term in $f(x)$. This completes the proof.
\end{proof}
\begin{remark}
The above proposition suggests that pruning is more powerful than merely flipping signs of a binary network. In fact, the same argument can be extended for binary networks of any fixed width $d$ and depth $l$ to show that pruned networks are more expressive. However, it does not quantify this difference in expressivity.
\end{remark}

%% file: binary_lth.bbl
\begin{thebibliography}{}

\bibitem[Barron, 1993]{barron1993universal}
Barron, A.~R. (1993).
\newblock Universal approximation bounds for superpositions of a sigmoidal
  function.
\newblock {\em IEEE Transactions on Information theory}, 39(3):930--945.

\bibitem[Blalock et~al., 2020]{mlsys2020_73}
Blalock, D., Gonzalez~Ortiz, J.~J., Frankle, J., and Guttag, J. (2020).
\newblock What is the state of neural network pruning?
\newblock In {\em Proceedings of Machine Learning and Systems 2020}, pages
  129--146.

\bibitem[Cheng et~al., 2019]{ChengEtAl19}
Cheng, Y., Wang, D., Zhou, P., and Zhang, T. (2019).
\newblock A {{Survey}} of {{Model Compression}} and {{Acceleration}} for {{Deep
  Neural Networks}}.
\newblock {\em arXiv:1710.09282 [cs]}.

\bibitem[Courbariaux et~al., 2015]{courbariaux2015binaryconnect}
Courbariaux, M., Bengio, Y., and David, J.-P. (2015).
\newblock Binaryconnect: Training deep neural networks with binary weights
  during propagations.
\newblock {\em arXiv preprint arXiv:1511.00363}.

\bibitem[Deng et~al., 2020]{deng2020model}
Deng, L., Li, G., Han, S., Shi, L., and Xie, Y. (2020).
\newblock Model compression and hardware acceleration for neural networks: A
  comprehensive survey.
\newblock {\em Proceedings of the IEEE}, 108(4):485--532.

\bibitem[Diffenderfer and Kailkhura, 2021]{diffenderfer2021multi}
Diffenderfer, J. and Kailkhura, B. (2021).
\newblock Multi-prize lottery ticket hypothesis: Finding accurate binary neural
  networks by pruning a randomly weighted network.
\newblock {\em arXiv preprint arXiv:2103.09377}.

\bibitem[Frankle and Carbin, 2018]{frankle2018lottery}
Frankle, J. and Carbin, M. (2018).
\newblock The lottery ticket hypothesis: Finding sparse, trainable neural
  networks.
\newblock In {\em International Conference on Learning Representations}.

\bibitem[Frankle et~al., 2020]{frankle2020linear}
Frankle, J., Dziugaite, G.~K., Roy, D., and Carbin, M. (2020).
\newblock Linear mode connectivity and the lottery ticket hypothesis.
\newblock In {\em International Conference on Machine Learning}, pages
  3259--3269. PMLR.

\bibitem[Han et~al., 2015a]{han2015deep}
Han, S., Mao, H., and Dally, W.~J. (2015a).
\newblock Deep compression: Compressing deep neural networks with pruning,
  trained quantization and huffman coding.
\newblock {\em arXiv preprint arXiv:1510.00149}.

\bibitem[Han et~al., 2015b]{han2015learning}
Han, S., Pool, J., Tran, J., and Dally, W.~J. (2015b).
\newblock Learning both weights and connections for efficient neural networks.
\newblock {\em arXiv preprint arXiv:1506.02626}.

\bibitem[Hanin, 2019]{hanin2019universal}
Hanin, B. (2019).
\newblock Universal function approximation by deep neural nets with bounded
  width and relu activations.
\newblock {\em Mathematics}, 7(10):992.

\bibitem[Hassibi and Stork, 1993]{Stork93}
Hassibi, B. and Stork, D.~G. (1993).
\newblock Second order derivatives for network pruning: {{Optimal Brain
  Surgeon}}.
\newblock In Hanson, S.~J., Cowan, J.~D., and Giles, C.~L., editors, {\em
  Advances in {{Neural Information Processing Systems}} 5}, pages 164--171.
  {Morgan-Kaufmann}.

\bibitem[He et~al., 2018]{he2018amc}
He, Y., Lin, J., Liu, Z., Wang, H., Li, L.-J., and Han, S. (2018).
\newblock Amc: Automl for model compression and acceleration on mobile devices.
\newblock In {\em Proceedings of the European Conference on Computer Vision
  (ECCV)}, pages 784--800.

\bibitem[He et~al., 2017]{he2017channel}
He, Y., Zhang, X., and Sun, J. (2017).
\newblock Channel pruning for accelerating very deep neural networks.
\newblock In {\em Proceedings of the IEEE International Conference on Computer
  Vision}, pages 1389--1397.

\bibitem[Hubara et~al., 2016]{hubara2016binarized}
Hubara, I., Courbariaux, M., Soudry, D., El-Yaniv, R., and Bengio, Y. (2016).
\newblock Binarized neural networks.
\newblock In {\em Proceedings of the 30th International Conference on Neural
  Information Processing Systems}, pages 4114--4122.

\bibitem[Hubara et~al., 2017]{hubara2017quantized}
Hubara, I., Courbariaux, M., Soudry, D., El-Yaniv, R., and Bengio, Y. (2017).
\newblock Quantized neural networks: Training neural networks with low
  precision weights and activations.
\newblock {\em The Journal of Machine Learning Research}, 18(1):6869--6898.

\bibitem[Kidger and Lyons, 2020]{kidger2020universal}
Kidger, P. and Lyons, T. (2020).
\newblock Universal approximation with deep narrow networks.
\newblock In {\em Conference on Learning Theory}, pages 2306--2327. PMLR.

\bibitem[Klusowski and Barron, 2018]{klusowski2018approximation}
Klusowski, J.~M. and Barron, A.~R. (2018).
\newblock Approximation by combinations of relu and squared relu ridge
  functions with $\ell^1$ and $\ell^0$ controls.
\newblock {\em IEEE Transactions on Information Theory}, 64(12):7649--7656.

\bibitem[LeCun et~al., 1990]{LeCunEtAl90}
LeCun, Y., Denker, J.~S., and Solla, S.~A. (1990).
\newblock Optimal {{Brain Damage}}.
\newblock In Touretzky, D.~S., editor, {\em Advances in {{Neural Information
  Processing Systems}} 2}, pages 598--605. {Morgan-Kaufmann}.

\bibitem[Levin et~al., 1994]{LevinEtAl94}
Levin, A.~U., Leen, T.~K., and Moody, J.~E. (1994).
\newblock Fast {{Pruning Using Principal Components}}.
\newblock In Cowan, J.~D., Tesauro, G., and Alspector, J., editors, {\em
  Advances in {{Neural Information Processing Systems}} 6}, pages 35--42.
  {Morgan-Kaufmann}.

\bibitem[Li et~al., 2016]{li2016pruning}
Li, H., Kadav, A., Durdanovic, I., Samet, H., and Graf, H.~P. (2016).
\newblock Pruning filters for efficient convnets.
\newblock {\em arXiv preprint arXiv:1608.08710}.

\bibitem[Malach et~al., 2020]{malach2020proving}
Malach, E., Yehudai, G., Shalev-Schwartz, S., and Shamir, O. (2020).
\newblock Proving the lottery ticket hypothesis: Pruning is all you need.
\newblock In {\em International Conference on Machine Learning}, pages
  6682--6691. PMLR.

\bibitem[Mozer and Smolensky, 1989]{mozer1989skeletonization}
Mozer, M.~C. and Smolensky, P. (1989).
\newblock Skeletonization: A technique for trimming the fat from a network via
  relevance assessment.
\newblock In {\em Advances in neural information processing systems}, pages
  107--115.

\bibitem[Orseau et~al., 2020]{orseau2020logarithmic}
Orseau, L., Hutter, M., and Rivasplata, O. (2020).
\newblock Logarithmic pruning is all you need.
\newblock {\em Advances in Neural Information Processing Systems}, 33.

\bibitem[Pensia et~al., 2020]{pensia2020optimal}
Pensia, A., Rajput, S., Nagle, A., Vishwakarma, H., and Papailiopoulos, D.
  (2020).
\newblock Optimal lottery tickets via subsetsum: Logarithmic
  over-parameterization is sufficient.
\newblock {\em arXiv preprint arXiv:2006.07990}.

\bibitem[Perekrestenko et~al., 2018]{perekrestenko2018universal}
Perekrestenko, D., Grohs, P., Elbr{\"a}chter, D., and B{\"o}lcskei, H. (2018).
\newblock The universal approximation power of finite-width deep relu networks.
\newblock {\em arXiv preprint arXiv:1806.01528}.

\bibitem[Ramanujan et~al., 2020]{RamanujanEtAl20}
Ramanujan, V., Wortsman, M., Kembhavi, A., Farhadi, A., and Rastegari, M.
  (2020).
\newblock What's {{Hidden}} in a {{Randomly Weighted Neural Network}}?
\newblock {\em arXiv:1911.13299 [cs]}.

\bibitem[Rastegari et~al., 2016]{rastegari2016xnor}
Rastegari, M., Ordonez, V., Redmon, J., and Farhadi, A. (2016).
\newblock Xnor-net: Imagenet classification using binary convolutional neural
  networks.
\newblock In {\em European conference on computer vision}, pages 525--542.
  Springer.

\bibitem[Scarselli and Tsoi, 1998]{scarselli1998universal}
Scarselli, F. and Tsoi, A.~C. (1998).
\newblock Universal approximation using feedforward neural networks: A survey
  of some existing methods, and some new results.
\newblock {\em Neural networks}, 11(1):15--37.

\bibitem[Simons and Lee, 2019]{simons2019review}
Simons, T. and Lee, D.-J. (2019).
\newblock A review of binarized neural networks.
\newblock {\em Electronics}, 8(6):661.

\bibitem[Stinchombe, 1989]{stinchombe1989universal}
Stinchombe, M. (1989).
\newblock Universal approximation using feed-forward networks with nonsigmoid
  hidden layer activation functions.
\newblock {\em Proc. IJCNN, Washington, DC, 1989}, pages 161--166.

\bibitem[Wang et~al., 2019]{WangEtAl19}
Wang, Y., Zhang, X., Xie, L., Zhou, J., Su, H., Zhang, B., and Hu, X. (2019).
\newblock Pruning from {{Scratch}}.
\newblock {\em arXiv:1909.12579 [cs]}.

\bibitem[Wen et~al., 2016]{wen2016learning}
Wen, W., Wu, C., Wang, Y., Chen, Y., and Li, H. (2016).
\newblock Learning structured sparsity in deep neural networks.
\newblock {\em arXiv preprint arXiv:1608.03665}.

\bibitem[Wu et~al., 2016]{wu2016quantized}
Wu, J., Leng, C., Wang, Y., Hu, Q., and Cheng, J. (2016).
\newblock Quantized convolutional neural networks for mobile devices.
\newblock In {\em Proceedings of the IEEE Conference on Computer Vision and
  Pattern Recognition}, pages 4820--4828.

\bibitem[Zhou et~al., 2019]{zhou2019deconstructing}
Zhou, H., Lan, J., Liu, R., and Yosinski, J. (2019).
\newblock Deconstructing lottery tickets: Zeros, signs, and the supermask.
\newblock {\em arXiv preprint arXiv:1905.01067}.

\bibitem[Zhu et~al., 2016]{zhu2016trained}
Zhu, C., Han, S., Mao, H., and Dally, W.~J. (2016).
\newblock Trained ternary quantization.
\newblock {\em arXiv preprint arXiv:1612.01064}.

\bibitem[Zhu and Gupta, 2017]{zhu2017prune}
Zhu, M. and Gupta, S. (2017).
\newblock To prune, or not to prune: exploring the efficacy of pruning for
  model compression.
\newblock {\em arXiv preprint arXiv:1710.01878}.

\end{thebibliography}
